\documentclass[sigconf]{acmart}
\usepackage{multirow}
\usepackage{stfloats}
\usepackage{graphicx}
\usepackage{subfig}
\usepackage{amsmath}
\usepackage{algorithm}
\usepackage{algorithmic}
\usepackage{amsthm}

\newtheorem{theorem}{Theorem}

\theoremstyle{definition}

\AtBeginDocument{%
  }

\begin{document}

\title{V-ABFT: Variance-Based Adaptive Threshold for Fault-Tolerant Matrix Multiplication in Mixed-Precision Deep Learning}

\author{Yiheng Gao}
\affiliation{
  \institution{Peking University}
  \city{Beijing}
  \country{China}
}

\author{Qin Hua}
\affiliation{
  \institution{The Chinese University of Hong Kong, Shenzhen}
  \city{Shenzhen}
  \country{China}
}

\author{Zizhong Chen}
\authornote{Corresponding author. Email: chenzizhong@cuhk.edu.cn.}
\affiliation{
  \institution{The Chinese University of Hong Kong, Shenzhen}
  \city{Shenzhen}
  \country{China}
}

\begin{abstract}
Algorithm-Based Fault Tolerance (ABFT) is widely adopted to detect silent data corruptions (SDCs) in matrix multiplication, a cornerstone operation in deep learning systems. However, existing threshold determination methods face critical challenges: analytical bounds are overly conservative, while probabilistic approaches like A-ABFT yield thresholds $160$--$4200\times$ larger than actual rounding errors. We present V-ABFT, a variance-based adaptive threshold algorithm that achieves tighter error bounds by directly modeling the verification difference. By leveraging statistical variance estimation, V-ABFT reduces the threshold-to-actual-error ratio to approximately $7$--$20\times$ for FP32/FP64 and $48$--$158\times$ for BF16, representing a \textbf{6--48$\times$ improvement} over A-ABFT while maintaining zero false positive rate across BF16, FP16, FP32, and FP64 precisions. Furthermore, we demonstrate that for fused-kernel ABFT implementations that verify before output quantization, low-precision GEMM can use FP32-level thresholds ($e_{\max} \approx 10^{-6}$), enabling \textbf{$\sim$1000$\times$ finer detection granularity} compared to offline verification with low-precision output ($e_{\max} \approx 10^{-3}$). We reproduce A-ABFT's experimental setup and validate our implementation against the original paper's results. Our method requires only $O(n)$ complexity using max/min/mean statistics, compared to A-ABFT's $O(pn)$ for finding $p$ largest values. Extensive experiments on synthetic data and real model weights (LLaMA-7B, GPT-2, ViT) demonstrate V-ABFT's effectiveness across diverse distributions. V-ABFT is platform-agnostic and has been integrated into fault-tolerant GEMM implementations on both NPUs and GPUs.
\end{abstract}


\maketitle

\section{Introduction}

The explosive growth of large language models (LLMs) has placed unprecedented demands on computational infrastructure reliability. Training models like GPT-4 involves approximately $2.15 \times 10^{25}$ floating-point operations~\cite{gpt4-cost}, with the vast majority being matrix multiplications in attention mechanisms and feed-forward layers. When deployed across thousands of accelerators running continuously for weeks, even rare transient hardware faults become statistically inevitable~\cite{sdc-at-scale}.

Soft errors, caused by cosmic rays or alpha particle strikes, manifest as bit-flips in computational logic that corrupt results without triggering hardware exceptions~\cite{soft-error-survey}. Unlike fail-stop errors that crash the system, these Silent Data Corruptions (SDCs) propagate through subsequent layers, potentially causing loss spikes, gradient explosions, or subtle accuracy degradation that wastes millions of dollars in compute resources~\cite{sdc-dl-impact}.

Algorithm-Based Fault Tolerance (ABFT)~\cite{huang-abraham} provides an elegant solution by encoding redundant checksums into matrix operands. For $C = A \times B$, ABFT verifies that row/column sums of $C$ match the products of encoded checksums, enabling error detection with minimal computational overhead. However, the key challenge lies in \textbf{threshold determination}: distinguishing legitimate floating-point rounding errors from actual soft errors.

Existing threshold methods fall into several categories, each with significant limitations:

\begin{itemize}
    \item \textbf{Experimental calibration}~\cite{exp-calibration}: Requires extensive offline testing and fails to generalize across data distributions.
    \item \textbf{Analytical bounds}~\cite{higham-accuracy}: Provides worst-case guarantees but yields thresholds $10^4$--$10^5\times$ larger than actual errors, missing most detectable faults.
    \item \textbf{Simplified Error Analysis (SEA)}~\cite{sea-abft}: Reduces computational overhead but still produces bounds $10^3$--$10^4\times$ actual errors.
    \item \textbf{A-ABFT}~\cite{a-abft}: Uses probabilistic modeling based on Benford's law for inner product error bounds, but yields loose thresholds when extended to matrix verification.
\end{itemize}

The emergence of mixed-precision training (BF16, FP16, FP8) exacerbates these challenges. Lower precision formats have larger unit roundoff errors ($u = 2^{-8}$ for BF16 vs. $u = 2^{-24}$ for FP32), making threshold determination significantly more difficult.

We propose \textbf{V-ABFT}, a variance-based adaptive threshold algorithm that addresses these limitations through three key contributions:

\begin{enumerate}
    \item \textbf{Tighter bounds via direct error modeling}: Instead of separately bounding checksum and row-sum errors, V-ABFT directly models the verification difference, reducing error accumulation and achieving bounds within $20\times$ of actual errors.

    \item \textbf{Distribution-agnostic design}: V-ABFT requires only bounded variance, not specific distributional forms, enabling robust operation across normal, uniform, truncated, and real model weight distributions.

    \item \textbf{Efficient variance estimation}: Using the relationship between extrema and variance, V-ABFT computes thresholds in $O(n)$ time using only max, min, and mean statistics.
\end{enumerate}

Extensive experiments demonstrate that V-ABFT achieves zero false positive rates across all tested distributions and precisions, while detecting 100\% of bit-flips in the top 5 exponent bits for BF16 matrices. The algorithm has been integrated into production fault-tolerant GEMM implementations on Ascend NPUs, incurring only 11.98\% average performance overhead.

\section{Background and Motivation}

\subsection{Soft Errors in Deep Learning Systems}

Soft errors arise when high-energy particles (cosmic rays, alpha particles) strike semiconductor circuits, generating electron-hole pairs that flip stored bits~\cite{soft-error-mechanism}. The critical charge threshold $Q_{crit}$ required to cause a flip has decreased dramatically with process scaling---from 28nm to 5nm, transistors become increasingly susceptible to lower-energy particles~\cite{scaling-vulnerability}.

In deep learning contexts, soft errors have heterogeneous impacts depending on the affected bit position within floating-point representations. Exponent bit-flips have catastrophic impact---a single flip in BF16's 8-bit exponent can change a value by factors of $2^{2^k}$, potentially causing overflow to INF/NaN that propagates through the entire model. Sign bit-flips have moderate impact, reversing gradient direction and disrupting training convergence. Mantissa bit-flips have minor impact, typically absorbed by model robustness as small noise perturbations.

Given this hierarchy, fault tolerance mechanisms should prioritize detecting exponent-level errors while tolerating small mantissa perturbations---a design principle central to V-ABFT's threshold formulation.

\subsection{ABFT for Matrix Multiplication}

For matrices $A \in \mathbb{R}^{M \times K}$ and $B \in \mathbb{R}^{K \times N}$ computing $C = A \times B$, ABFT encodes redundant checksums into the operands to enable error detection and correction. The encoding uses two types of checksum vectors:

\paragraph{Checksum encoding.} The row checksum encoding for matrix $B$ uses:
\begin{equation}
\mathbf{r}_1 = [1, 1, \ldots, 1]^T, \quad \mathbf{r}_2 = [1, 2, 3, \ldots, N]^T
\end{equation}
where $\mathbf{r}_1$ is the \textbf{all-ones vector} for error detection, and $\mathbf{r}_2$ is the \textbf{position-weighted vector} for error localization. Similarly, column checksums use $\mathbf{c}_1 = [1, 1, \ldots, 1]$ and $\mathbf{c}_2 = [1, 2, \ldots, M]$.

The encoded matrices are:
\begin{equation}
B^r = \begin{bmatrix} B & B\mathbf{r}_1 & B\mathbf{r}_2 \end{bmatrix}, \quad
A^c = \begin{bmatrix} A \\ \mathbf{c}_1 A \\ \mathbf{c}_2 A \end{bmatrix}
\end{equation}

Computing $C^f = A^c \cdot B^r$ yields:
\begin{equation}
\label{eq:abft-encoded}
C^f = \begin{bmatrix}
C & C^{r_1} & C^{r_2} \\
C^{c_1} & \cdot & \cdot \\
C^{c_2} & \cdot & \cdot
\end{bmatrix}
\end{equation}
where $C^{r_1} = AB\mathbf{r}_1$, $C^{r_2} = AB\mathbf{r}_2$ are row checksums, and $C^{c_1} = \mathbf{c}_1 AB$, $C^{c_2} = \mathbf{c}_2 AB$ are column checksums.

\paragraph{Error detection.} To verify correctness, the actual row sums of $C$ are computed:
\begin{equation}
C^{r_1\prime} = C \cdot \mathbf{r}_1 = \sum_{j=0}^{N-1} C_{:,j}, \quad
C^{r_2\prime} = C \cdot \mathbf{r}_2 = \sum_{j=0}^{N-1} (j+1) \cdot C_{:,j}
\end{equation}

An error is detected in row $i$ if:
\begin{equation}
\label{eq:abft-check}
\left| C^{r_1}[i] - C^{r_1\prime}[i] \right| > \text{Threshold}_i
\end{equation}

\paragraph{Error localization and correction.} We derive the localization and correction formulas under the single-event upset (SEU) model, where at most one error occurs per row within each detection cycle.

Let $C_{\text{ref}}$ denote the correct (fault-free) result. Define the element-wise error:
\begin{equation}
\delta_k := C[i][k] - C_{\text{ref}}[i][k], \quad 0 \leq k \leq N-1
\end{equation}
Under the SEU model, ideally $\delta_j \neq 0$ for exactly one position $j$, with $\delta_k = 0$ for all $k \neq j$. In practice, floating-point rounding introduces small nonzero $\delta_k$ at all positions, but soft errors produce $|\delta_j|$ significantly larger than rounding noise.

\textbf{Checksum differences.} The pre-computed checksum $C^{r_1}[i] = \sum_k C_{\text{ref}}[i][k]$ and post-computation checksum $C^{r_1\prime}[i] = \sum_k C[i][k]$ differ by the total error:
\begin{equation}
D_1 := C^{r_1\prime}[i] - C^{r_1}[i] = \sum_{k} \delta_k \approx \delta_j
\end{equation}
Similarly, for the position-weighted checksum with encoding $w(k)$:
\begin{equation}
D_2 := C^{r_2\prime}[i] - C^{r_2}[i] = \sum_{k} w(k) \cdot \delta_k \approx w(j) \cdot \delta_j
\end{equation}

\textbf{Localization.} The error position is recovered by taking the ratio:
\begin{equation}
\label{eq:error-localization}
\frac{D_2}{D_1} = \frac{\sum_{k} w(k) \cdot \delta_k}{\sum_{k} \delta_k} \approx w(j) \quad \Rightarrow \quad j = w^{-1}\left(\frac{D_2}{D_1}\right)
\end{equation}
For linear encoding $w(k) = k+1$, this simplifies to $j = D_2/D_1 - 1$.

\textbf{Correction.} The error magnitude $\Delta := D_1 \approx \delta_j$ directly gives the correction:
\begin{equation}
C[i][j] \leftarrow C[i][j] - \Delta \approx C_{\text{ref}}[i][j]
\end{equation}

This enables \textbf{online error correction} without recomputation. The approximations become exact when rounding errors are negligible compared to the soft error magnitude.

\paragraph{Verification difference and threshold.} Due to floating-point arithmetic's non-associativity, the two computation paths ($C^{r_1}$ vs. $C^{r_1\prime}$) yield different rounding errors even without faults. The \textbf{verification difference} is:
\begin{equation}
E_i = C^{r_1}[i] - C^{r_1\prime}[i]
\end{equation}

The threshold must be large enough to avoid false positives from rounding errors, yet small enough to detect actual soft errors. This fundamental tension motivates the need for tight, data-adaptive thresholds.

\section{V-ABFT Algorithm Design}

\subsection{Floating-Point Error Model}

We adopt a ``black-box'' model for floating-point accumulation that accounts for hardware-specific computation patterns. For a single operation:
\begin{equation}
fl(x \circ y) = (x \circ y)(1 + \delta u), \quad |\delta| \leq 1
\end{equation}
where $u$ is the unit roundoff and $\circ \in \{+, \times\}$.

For an accumulation sequence of effective depth $s$, the cumulative error factor is $(1 + \delta u)^s$. In practice, $s$ depends on hardware microarchitecture (e.g., tile-based accumulation, tree reduction) and may differ between the two computation paths in Eq.~\eqref{eq:abft-check}.

\subsection{Direct Verification Difference Modeling}

Unlike A-ABFT's approach of separately bounding errors, V-ABFT directly models the verification difference $E$:
\begin{equation}
E = \left| fl\left( \sum_n fl\left( \sum_k A_{mk} B_{kn} \right) \right) - fl\left( \sum_k A_{mk} fl\left( \sum_n B_{kn} \right) \right) \right|
\end{equation}

Expanding using the floating-point model with effective error factors $e_{kn}$:
\begin{equation}
E = \left| \sum_n \sum_k e_{kn} A_{mk} B_{kn} \right|
\end{equation}
where $e_{kn} = (1+\delta u)^{s_1} - (1+\delta u)^{s_2}$ captures the path difference.

\subsection{Statistical Expansion}

Rather than using worst-case triangle inequality bounds (which yield overly conservative thresholds), we employ statistical modeling. We decompose matrix elements using row-wise statistics:
\begin{align}
A_{mk} &= \mu_{Am} + \sigma_{Am} \cdot a_{mk}, \quad a_{mk} \sim F_a \\
B_{kn} &= \mu_{Bk} + \sigma_{Bk} \cdot b_{kn}, \quad b_{kn} \sim F_b
\end{align}
where $\mu_{Am}, \sigma_{Am}$ are the mean and standard deviation of the $m$-th row of $A$, $\mu_{Bk}, \sigma_{Bk}$ are those of the $k$-th row of $B$, and $F_a, F_b$ are unit-variance distributions.

Substituting into the verification error expression:
\begin{equation}
E = \left| \sum_n \sum_k e_{kn} (\mu_{Am} + \sigma_{Am} a_{mk})(\mu_{Bk} + \sigma_{Bk} b_{kn}) \right|
\end{equation}

We first sum over the $N$ columns of $B$. Since $\mu_{Bk}$ and $\sigma_{Bk}$ are constants for fixed $k$, and applying the Central Limit Theorem to the sum $\sum_n e_{kn} b_{kn}$, we define:
\begin{equation}
\alpha_k = \frac{\sum_n e_{kn}}{N}, \quad \beta_k = \sqrt{\frac{\sum_n e_{kn}^2}{N}}
\end{equation}

The column sum of $B$'s random component satisfies:
\begin{align}
\sum_n e_{kn} \mu_{Bk} &= N \alpha_k \mu_{Bk} \\
\sum_n e_{kn} \sigma_{Bk} b_{kn} &= \sqrt{N} \sigma_{Bk} \beta_k b'_k
\end{align}
where $b'_k \sim F'_b$ is a unit-variance random variable representing the aggregated fluctuation of $B$'s $k$-th row.

Substituting back and expanding the product of $(A + a)(B + b)$ terms:
\begin{align}
E = \Bigg| & \underbrace{N \mu_{Am} \sum_k \alpha_k \mu_{Bk}}_{\text{Term 1: Bias}} + \underbrace{\sqrt{N} \mu_{Am} \sum_k \sigma_{Bk} \beta_k b'_k}_{\text{Term 2: Random B}} \nonumber \\
& + \underbrace{N \sigma_{Am} \sum_k \alpha_k \mu_{Bk} a_{mk}}_{\text{Term 3: Random A}} + \underbrace{\sqrt{N} \sigma_{Am} \sum_k \sigma_{Bk} \beta_k a_{mk} b'_k}_{\text{Term 4: Interaction}} \Bigg|
\end{align}

\paragraph{Physical interpretation.} The four terms have distinct origins: Term 1 (Bias) is the deterministic component from mean values, proportional to $N$; Term 2 (Random B) captures fluctuations from $B$'s row variance, growing as $\sqrt{N}$; Term 3 (Random A) captures fluctuations from $A$'s row variance, growing as $\sqrt{K}$ after summing over $k$; and Term 4 (Interaction) represents second-order fluctuations when both $A$ and $B$ deviate from their means, which dominates when matrices are zero-mean.

\subsection{Threshold Formula Derivation}

To derive an engineering threshold, we apply the triangle inequality to separate deterministic and random components:
\begin{align}
E \leq & \left| \sum_k \alpha_k N \mu_{Am} \mu_{Bk} \right| \nonumber \\
& + \left| \sum_k \beta_k \sqrt{N} \mu_{Am} \sigma_{Bk} b'_k + \sum_k \alpha_k N \sigma_{Am} a_{mk} \mu_{Bk} \right| \nonumber \\
& + \left| \sum_k \beta_k \sqrt{N} \sigma_{Am} a_{mk} \sigma_{Bk} b'_k \right|
\end{align}

For the random terms, we use the variance addition property for independent random variables: $\text{Var}(\sum X_i) = \sum \text{Var}(X_i)$. Terms 2 and 3 are independent (one depends on $b'_k$, the other on $a_{mk}$), so their variances add. We apply a confidence multiplier $c_\sigma$ (corresponding to the desired coverage probability).

Introducing a uniform bound $e_{\max} \geq \max\{|\alpha_k|, |\beta_k|\}$ to factor out the error coefficients:

\begin{equation}
\boxed{
\begin{aligned}
T_{\text{bound}} &= e_{\max} \cdot \Bigg( \underbrace{N |\mu_{Am}| \sum_k |\mu_{Bk}|}_{\text{Deterministic}} \\
&+ c_\sigma \sqrt{\underbrace{N \mu_{Am}^2 \sum_k \sigma_{Bk}^2 + N^2 \sigma_{Am}^2 \sum_k \mu_{Bk}^2}_{\text{Variance of Terms 2 and 3}}} \\
&+ c_\sigma \underbrace{\sqrt{N} \sigma_{Am} \sqrt{\sum_k \sigma_{Bk}^2}}_{\text{Std of Term 4}} \Bigg)
\end{aligned}
}
\end{equation}

where $c_\sigma$ is a confidence coefficient (we use $c_\sigma = 2.5$ for approximately 99\% coverage under Gaussian assumptions; this can be adjusted based on desired false positive rate tolerance).

\subsection{Efficient Variance Estimation}

Computing full variance requires $O(n)$ squared differences per row. To reduce overhead, we use the \textbf{extrema-variance bound}.

\begin{theorem}[Extrema-Variance Bound]
For any sequence $x_1, x_2, \ldots, x_n$ with mean $\mu$, maximum $m = \max_i x_i$, and minimum $l = \min_i x_i$:
\begin{equation}
\sigma^2 \leq (m - \mu)(\mu - l)
\end{equation}
\end{theorem}

\begin{proof}
Define $f(x) = (x - \mu)^2$. Since $f$ is convex, for any $x \in [l, m]$:
$$f(x) \leq \frac{m - x}{m - l} f(l) + \frac{x - l}{m - l} f(m)$$
Substituting and simplifying yields $(x - \mu)^2 \leq (m - \mu)(\mu - l)$ for all $x$. Since variance is the average of $(x_i - \mu)^2$, the bound follows.
\end{proof}

\paragraph{Properties.} This bound is tight when all values cluster at the extremes (e.g., half at $l$, half at $m$), while for well-spread distributions (e.g., Gaussian), it overestimates variance by a constant factor---a conservative property that is safe for threshold computation. Most importantly, it requires only max, min, and mean statistics, all computable in a single pass with $O(n)$ complexity.

\subsection{Determination of $e_{\max}$}

The effective error coefficient $e_{\max}$ captures the maximum relative rounding error from the two computation paths. We define $e_{\max}$ empirically as:
\begin{equation}
e_{\max} = \max \frac{|E|}{|\text{checksum}|}
\end{equation}
where $E$ is the actual verification difference and checksum $= A \cdot B^r$.

\paragraph{Why not use the theoretical bound as denominator?} An alternative definition would be $e_{\max} = |E| / T_{\text{bound}}$, using our theoretical bound (Eq.~5) as the denominator. However, we chose $|\text{checksum}|$ for robustness: the theoretical bound $T_{\text{bound}}$ relies on variance estimates using the extrema-variance inequality, which is loose for well-behaved matrices (yielding $e_{\max} \approx 0.1$--$0.2u$) but tight for adversarial distributions where data clusters at extremes. Using $|\text{checksum}|$ as denominator produces larger $e_{\max}$ values ($\approx 2u$), providing a safety margin that accommodates distribution variations. Additionally, the ratio $|E|/|\text{checksum}|$ directly measures the relative verification error, independent of our modeling assumptions.

\paragraph{Theoretical analysis.} Modern accelerators (NPUs, GPUs with Tensor Cores) typically employ a ``mixed-precision accumulation'' strategy for low-precision GEMM: inputs are in low precision (BF16/FP16/FP8), but internal multiply-accumulate operations use higher precision (FP32), with rounding only at the final output. For BF16/FP16, the entire accumulation occurs in FP32, so each computation path incurs only one rounding error at output, yielding $e_{\max} \approx 2u$ independent of matrix dimension $K$. For FP8 (E4M3/E5M2), despite the extremely low input precision, the computation follows the same pattern (FP8 inputs $\to$ FP32 accumulation $\to$ FP16 output), so $e_{\max}$ is determined by the output precision (FP16), not the input precision---our measurements confirm $e_{\max} \approx 2u_{\text{FP16}} \approx 10^{-3}$ for both formats. For FP32, accumulation occurs in FP32 throughout, incurring rounding at each addition, so the cumulative error grows with accumulation depth, yielding $e_{\max} \propto O(\sqrt{K})$ or $O(\log K)$ depending on the reduction tree structure.

\paragraph{Unified view for mixed-precision verification.} This analysis reveals a key insight: \textbf{for V-ABFT threshold computation, what matters is the accumulation and output precision, not the input precision}. Consequently, FP8 training (increasingly popular for LLMs) can reuse the same $e_{\max}$ as FP16 since both use FP32 accumulation with FP16 output; when verifying mixed-precision GEMM (e.g., FP8$\times$FP8$\to$FP16), the threshold should be computed based on the output precision; and a single $e_{\max}$ value covers multiple input precisions sharing the same accumulation strategy, reducing calibration burden.

\paragraph{Verification output precision: Offline vs. Online ABFT.} A critical design choice is \emph{when} to perform verification relative to the output quantization step, which dramatically affects detection sensitivity. In \textbf{Offline ABFT}, verification is performed after the GEMM result is written back to memory in low precision (FP16/BF16), so the checksum $Ce$ inherits the low-precision rounding error with $e_{\max} \approx 2u_{\text{output}} \approx 10^{-3}$ for FP16. In contrast, \textbf{Online (Fused Kernel) ABFT} performs verification before quantization while the result is still in the FP32 accumulator, achieving $e_{\max} \approx 10^{-6}$---a \textbf{1000$\times$ finer detection granularity}.

\paragraph{Practical recommendations.} For offline ABFT (post-hoc verification), use $e_{\max} \approx 2u_{\text{output}}$, which is suitable for debugging or spot-checking. For fused-kernel ABFT (inline verification), use $e_{\max} \approx 10^{-6}$ (FP32 level), which requires kernel modification to access FP32 accumulators before quantization but achieves $\sim$1000$\times$ finer detection granularity.

\paragraph{Empirical validation.} We validated $e_{\max}$ scaling behavior across multiple platforms with extensive measurements.

\subparagraph{NPU (Ascend 910B).} Table~\ref{tab:emax-scaling-npu} summarizes measured $e_{\max}$ on NPU across matrix sizes $N \in \{128, 256, \ldots, 8192\}$ with 2000 trials per configuration:

\begin{table}[tb]
\centering
\caption{Measured $e_{\max}$ scaling behavior on Ascend 910B NPU}
\label{tab:emax-scaling-npu}
\vspace{-2mm}
\resizebox{0.48\textwidth}{!}{
\begin{tabular}{|c|c|c|c|c|}
\hline
Precision & $u$ & $e_{\max}$ (recommended) & $e_{\max}/u$ & Scales with $N$? \\
\hline
BF16 & $2^{-8}$ & $8 \times 10^{-3}$ & $\sim 2.0$ & No \\
FP16 & $2^{-11}$ & $1 \times 10^{-3}$ & $\sim 2.0$ & No \\
FP32 & $2^{-24}$ & $2 \times 10^{-6} \sqrt{N/1024}$ & $\sim 34\sqrt{N/1024}$ & Yes ($\propto \sqrt{N}$) \\
\hline
\end{tabular}
}
\end{table}

\subparagraph{CPU and GPU (H100).} We measured $e_{\max}$ for $N \times N$ square matrices on CPU (Intel Xeon) and GPU (NVIDIA H100) with 500--2000 trials per size. Table~\ref{tab:emax-scaling-cpu-gpu} summarizes the results:

\begin{table}[tb]
\centering
\caption{Measured $e_{\max}$ scaling behavior on CPU and GPU (H100)}
\label{tab:emax-scaling-cpu-gpu}
\vspace{-2mm}
\resizebox{0.48\textwidth}{!}{
\begin{tabular}{|c|c|c|c|c|c|}
\hline
Platform & Precision & $e_{\max}/u$ range & CV & $R^2$ ($\sqrt{N}$) & Scaling \\
\hline
CPU & FP64 & 3.6--4.8 & 11.7\% & 0.84 & $\approx$ constant \\
CPU & FP32 & 5.0--6.1 & 8.4\% & 0.70 & $\approx$ constant \\
GPU & FP64 & 2.7--7.1 & 34.0\% & 0.88 & $\propto \sqrt{N}$ \\
GPU & FP32 & 2.6--6.0 & 34.0\% & 0.94 & $\propto \sqrt{N}$ \\
GPU & BF16 & $\sim$2.0 & 5.1\% & -- & constant \\
GPU & FP16 & $\sim$2.0 & 4.8\% & -- & constant \\
GPU & FP8 (E4M3/E5M2) & $\sim$2.0$^*$ & 6.5\% & -- & constant \\
\hline
\multicolumn{6}{l}{\small $^*$Relative to $u_{\text{FP16}}$, not $u_{\text{FP8}}$; see text.}
\end{tabular}
}
\end{table}

Key observations reveal platform-specific behavior. On CPU, $e_{\max}$ is approximately constant (CV $< 12\%$) around $4$--$6u$, with recommended values $e_{\max}^{\text{FP64}} = 6 \times 10^{-16}$ and $e_{\max}^{\text{FP32}} = 4 \times 10^{-7}$. On GPU for high precision (FP32/FP64), $e_{\max}$ grows with $\sqrt{N}$ (CV $\approx 34\%$, $R^2 > 0.88$), with fitted formulas $e_{\max}^{\text{FP64}} = 8.3 \times 10^{-18} \sqrt{N} + 2.1 \times 10^{-16}$ and $e_{\max}^{\text{FP32}} = 4.1 \times 10^{-9} \sqrt{N} + 9.4 \times 10^{-8}$. For GPU low precision (BF16/FP16/FP8), $e_{\max} \approx 2u_{\text{output}}$ (constant) due to FP32 internal accumulation, with FP8's effective $e_{\max} \approx 1 \times 10^{-3}$ equaling the FP16 value.

The different scaling behaviors reflect underlying hardware differences: CPUs with FMA instructions achieve near-optimal rounding, while GPU tensor cores use different accumulation strategies. Notably, all low-precision formats (BF16/FP16/FP8) share the same $e_{\max}$ behavior when they use FP32 accumulation with FP16 output.

\paragraph{Practical calibration protocol.} For deployment on new hardware, we recommend a one-time calibration:
\begin{enumerate}
    \item Generate positive matrices with $|{\mathcal{N}(1, 1)}|$ elements (positive to avoid cancellation in denominators)
    \item Compute relative verification error across 100k trials at representative sizes
    \item Set $e_{\max}$ to the observed maximum with 20\% safety margin
\end{enumerate}

Reference values: GPU (H100) BF16/FP16/FP8: $1 \times 10^{-3}$; Ascend 910B BF16: $8 \times 10^{-3}$, FP16: $1 \times 10^{-3}$, FP32: $2 \times 10^{-6} \sqrt{K/1024}$.

\section{Theoretical Comparison with A-ABFT}

\subsection{A-ABFT Methodology Review}

A-ABFT~\cite{a-abft} provides probabilistic error bounds for \textbf{inner product} operations. For computing $s = \sum_{k=1}^{n} a_k \cdot b_k$, the rounding error standard deviation is bounded by:
\begin{equation}
\sigma(\Delta s_n) \leq \sqrt{\frac{n(n+1)(n+0.5)+2n}{24}} \cdot 2^{-t} \cdot y
\label{eq:aabft-inner}
\end{equation}
where $n$ is the inner product length, $t$ is the mantissa bits (53 for FP64, 23 for FP32), and $y = \max|a_k \cdot b_k|$ is the maximum product magnitude.

The threshold for detecting errors is set to $3\sigma$ (approximately 99.7\% confidence under normal assumptions).

\paragraph{A-ABFT for matrix verification.} A-ABFT uses partitioned encoding~\cite{partitioned-encoding}, dividing matrices into sub-blocks and computing checksums per block. The parameter $y$ is determined by finding the $p$ largest absolute values in each row/column. For $n \times n$ matrices with elements in $[-1, 1]$, the effective $y \approx 21$ based on empirical calibration from the original paper's experiments on NVIDIA K20C GPU.

\subsection{Limitations of A-ABFT}

While A-ABFT represents a significant advance in autonomous threshold determination, it has several limitations when applied to matrix verification:

\paragraph{1. Designed for inner products, not matrix verification.} A-ABFT provides error bounds for single inner product operations. Extending it to ABFT checksum verification requires bounding checksum errors and row-sum errors independently, then combining them conservatively. This ignores potential error cancellation between computation paths, inflating the final threshold unnecessarily.

\paragraph{2. Conservative error accumulation.} A-ABFT's threshold grows with $\sqrt{n(n+1)(n+0.5)+2n} \sim O(n^{1.5})$, leading to increasingly loose bounds for larger matrices.

\paragraph{3. Worst-case element bounds.} The parameter $y = \max|a_k \cdot b_k|$ uses worst-case element magnitudes as the scaling factor, which becomes increasingly loose for large matrices with varied element magnitudes.

\paragraph{4. Distribution-dependent $y$ estimation.} Computing $y$ requires finding the $p$ largest absolute values, adding $O(pn)$ complexity, and the optimal $p$ depends on data distribution characteristics.

\subsection{Bound Tightness Comparison}

\paragraph{Verification difference level.} Table~\ref{tab:comparison} summarizes the key differences between A-ABFT and V-ABFT.

\begin{table}[tb]
\centering
\caption{Comparison of V-ABFT and A-ABFT for Verification}
\label{tab:comparison}
\vspace{-2mm}
\resizebox{0.48\textwidth}{!}{
\begin{tabular}{|l|c|c|}
\hline
\textbf{Aspect} & \textbf{A-ABFT} & \textbf{V-ABFT} \\
\hline
Error modeling & Per-operation bounds & Direct verification diff. \\
\hline
Distribution assumption & Benford's law (mantissa) & Bounded variance only \\
\hline
Bound tightness (FP64) & $160$--$375 \times$ actual & $18$--$31 \times$ actual \\
\hline
Bound tightness (FP32) & $323$--$751 \times$ actual & $9$--$15 \times$ actual \\
\hline
Complexity & $O(pn)$ for $p$ largest values & $O(n)$ for max/min/mean \\
\hline
Precision support & Primarily FP64 & BF16/FP16/FP32/FP64 \\
\hline
\end{tabular}
}
\end{table}

V-ABFT achieves $11$--$23\times$ tighter bounds for FP64, $16$--$48\times$ for FP32, and $6$--$27\times$ for BF16 compared to A-ABFT, while maintaining zero false positive rate.

\subsection{Complexity Analysis}

For a matrix with $K$ elements per row, A-ABFT requires finding the $p$ largest absolute values, costing $O(pK)$ or $O(K \log p)$ with heap-based selection, while V-ABFT requires only one pass for max, min, and sum, costing $O(K)$.

\section{Implementation}

\subsection{Algorithm Pseudocode}

\begin{algorithm}[tb]
\caption{V-ABFT Threshold Computation}
\label{alg:vabft}
\begin{algorithmic}[1]
\REQUIRE Row $m$ of matrix $A$, matrix $B$, precision parameters $e_{\max}$, $c_\sigma$
\ENSURE Threshold $T_m$ for row $m$

\STATE $\mu_A \leftarrow \text{mean}(A[m,:])$
\STATE $\sigma_A^2 \leftarrow (\max(A[m,:]) - \mu_A)(\mu_A - \min(A[m,:]))$

\FOR{$k = 0$ to $K-1$}
    \STATE $\mu_{Bk} \leftarrow \text{mean}(B[k,:])$
    \STATE $\sigma_{Bk}^2 \leftarrow (\max(B[k,:]) - \mu_{Bk})(\mu_{Bk} - \min(B[k,:]))$
\ENDFOR

\STATE $T_{\text{det}} \leftarrow N \cdot |\mu_A| \cdot \sum_k |\mu_{Bk}|$
\STATE $T_{\text{var23}} \leftarrow c_\sigma \sqrt{N \sum_k \mu_A^2 \sigma_{Bk}^2 + N^2 \sigma_A^2 \sum_k \mu_{Bk}^2}$
\STATE $T_{\text{var4}} \leftarrow c_\sigma \sqrt{N} \cdot \sigma_A \cdot \sqrt{\sum_k \sigma_{Bk}^2}$

\RETURN $e_{\max} \cdot (T_{\text{det}} + T_{\text{var23}} + T_{\text{var4}})$
\end{algorithmic}
\end{algorithm}

\subsection{Integration with Block-wise ABFT}

For large matrices, block-wise ABFT partitions computation into tiles, reducing per-block rounding errors. V-ABFT integrates naturally with this approach:

\begin{enumerate}
    \item Compute row statistics (max, min, mean) for each block of $A$ and $B$
    \item Apply V-ABFT threshold formula per block
    \item Aggregate block checksums for final verification
\end{enumerate}

On Ascend 910B with typical tile sizes $(M, K, N) = (128, 1024, 256)$, this achieves reliable detection while keeping overhead within the GEMM pipeline's slack.

\section{Evaluation}

\subsection{Experimental Setup}

\paragraph{Hardware.} We evaluate V-ABFT on three platforms:
\begin{itemize}
    \item \textbf{NPU}: Ascend 910B with CANN 8.2
    \item \textbf{GPU}: NVIDIA H100 with CUDA 12.0
    \item \textbf{CPU}: Intel Xeon (H100 server) with PyTorch 2.1
\end{itemize}

\paragraph{Distributions tested.}
\begin{itemize}
    \item $\mathcal{N}(10^{-6}, 1)$: Near-zero mean, common in normalized activations
    \item $\mathcal{N}(1, 1)$: Non-zero mean, stress test for A-ABFT
    \item $\mathcal{U}(-1, 1)$: Uniform distribution
    \item Truncated $\mathcal{N}(0, 1)$ in $[-1, 1]$: Bounded distributions
\end{itemize}

\paragraph{Error injection.} Single bit-flips in exponent positions (bits 7--15 for BF16), both $0 \to 1$ and $1 \to 0$ directions.

\paragraph{Metrics.} False Positive Rate (FPR): incorrect fault alarms on clean data. Detection Rate (DR): correct identification of injected errors. Tightness: threshold divided by actual verification error (lower is better).

\subsection{Threshold Tightness on CPU and GPU}

We compare V-ABFT and A-ABFT threshold tightness on CPU platforms using PyTorch, following the experimental setup of A-ABFT~\cite{a-abft}: matrices are drawn from $\mathcal{U}(-1,1)$ uniform distribution. For FP64 testing, we use the mpmath library (100 decimal places precision) as baseline to accurately measure the true verification difference; for FP32, we use FP64 as the baseline precision.

\paragraph{Reproduction of A-ABFT.} We faithfully reproduce the A-ABFT algorithm using the formula from~\cite{a-abft}: $\sigma(\Delta s_n) \leq \sqrt{(n(n+1)(n+0.5)+2n)/24} \cdot 2^{-t} \cdot y$, with $y=21$ (the empirical value from the original paper corresponding to block size $\approx$150 partitioned encoding) and threshold $= 3\sigma$. Our reproduction closely matches the original A-ABFT paper's Table II results: at $512 \times 512$ FP64, our A-ABFT threshold is $0.99\times$ the paper's value ($1.66 \times 10^{-11}$ vs. $1.68 \times 10^{-11}$), and similar ratios ($0.91$--$0.99\times$) hold across all tested sizes ($512$--$2048$). This validates the correctness of our comparison methodology.

Table~\ref{tab:tightness-fp64} shows FP64 results with mpmath high-precision baseline.

\begin{table}[tb]
\centering
\caption{Threshold Tightness (FP64, $\mathcal{U}(-1,1)$, mpmath baseline, 20 trials)}
\label{tab:tightness-fp64}
\vspace{-2mm}
\resizebox{0.48\textwidth}{!}{
\begin{tabular}{|c|c|c|c|c|c|}
\hline
\textbf{Size} & \textbf{Actual Diff} & \textbf{A-ABFT} & \textbf{V-ABFT} & \textbf{A-Tight} & \textbf{V-Tight} \\
\hline
$128 \times 128$ & 1.27e-14 & 2.08e-12 & 1.87e-13 & 164$\times$ & 15$\times$ \\
$256 \times 256$ & 3.64e-14 & 5.87e-12 & 4.34e-13 & 161$\times$ & 12$\times$ \\
$512 \times 512$ & 1.04e-13 & 1.66e-11 & 1.01e-12 & 160$\times$ & 10$\times$ \\
$1024 \times 1024$ & 2.95e-13 & 4.68e-11 & 2.41e-12 & 159$\times$ & 8$\times$ \\
$2048 \times 2048$ & 8.22e-13 & 1.32e-10 & 5.96e-12 & 161$\times$ & 7$\times$ \\
\hline
\end{tabular}
}
\end{table}

Table~\ref{tab:tightness-fp32} shows FP32 results with FP64 baseline.

\begin{table}[tb]
\centering
\caption{Threshold Tightness (FP32, $\mathcal{U}(-1,1)$, FP64 baseline, 100 trials)}
\label{tab:tightness-fp32}
\vspace{-2mm}
\resizebox{0.48\textwidth}{!}{
\begin{tabular}{|c|c|c|c|c|c|}
\hline
\textbf{Size} & \textbf{Actual Diff} & \textbf{A-ABFT} & \textbf{V-ABFT} & \textbf{A-Tight} & \textbf{V-Tight} \\
\hline
$128 \times 128$ & 6.96e-6 & 2.23e-3 & 9.19e-5 & 321$\times$ & 13$\times$ \\
$256 \times 256$ & 1.02e-5 & 6.30e-3 & 2.09e-4 & 616$\times$ & 20$\times$ \\
$512 \times 512$ & 2.81e-5 & 1.78e-2 & 4.93e-4 & 633$\times$ & 18$\times$ \\
$1024 \times 1024$ & 1.57e-4 & 5.03e-2 & 1.19e-3 & 321$\times$ & 8$\times$ \\
$2048 \times 2048$ & 4.41e-4 & 1.42e-1 & 2.94e-3 & 322$\times$ & 7$\times$ \\
\hline
\end{tabular}
}
\end{table}

Table~\ref{tab:tightness-bf16} shows BF16 results on GPU with computed $y = \max|A| \times \max|\sum_j B_{kj}|$.

\begin{table}[tb]
\centering
\caption{Threshold Tightness (BF16, $\mathcal{U}(0,1)$, GPU H100, 100 trials)}
\label{tab:tightness-bf16}
\vspace{-2mm}
\resizebox{0.48\textwidth}{!}{
\begin{tabular}{|c|c|c|c|c|c|}
\hline
\textbf{Size} & \textbf{Actual Diff} & \textbf{A-ABFT} & \textbf{V-ABFT} & \textbf{A-Tight} & \textbf{V-Tight} \\
\hline
$128 \times 128$ & 8.41e-1 & 2.53e+2 & 4.05e+1 & 300$\times$ & 48$\times$ \\
$256 \times 256$ & 2.27e+0 & 1.38e+3 & 1.53e+2 & 611$\times$ & 67$\times$ \\
$512 \times 512$ & 6.24e+0 & 7.66e+3 & 5.85e+2 & 1229$\times$ & 94$\times$ \\
$1024 \times 1024$ & 1.83e+1 & 4.25e+4 & 2.27e+3 & 2324$\times$ & 124$\times$ \\
$2048 \times 2048$ & 5.59e+1 & 2.37e+5 & 8.86e+3 & 4233$\times$ & 158$\times$ \\
\hline
\end{tabular}
}
\end{table}

Key observations demonstrate V-ABFT's consistent advantage: for FP64, V-ABFT achieves $7$--$15\times$ tightness compared to A-ABFT's $159$--$164\times$, representing a 11--23$\times$ improvement; for FP32, V-ABFT achieves $7$--$20\times$ tightness compared to A-ABFT's $321$--$633\times$, representing a 16--48$\times$ improvement; for BF16, V-ABFT achieves $48$--$158\times$ tightness compared to A-ABFT's $300$--$4233\times$, representing a 6--27$\times$ improvement. Notably, A-ABFT's tightness degrades with matrix size due to its $O(n^{1.5})$ variance coefficient, while V-ABFT remains stable. Both methods achieve 0\% false positive rate across all configurations when A-ABFT uses computed $y$ values.

\subsection{$e_{\max}$ Calibration Results}

Table~\ref{tab:emax-platforms} shows measured $e_{\max}$ values across platforms, validating our theoretical analysis of hardware-dependent error accumulation.

\begin{table}[tb]
\centering
\caption{Recommended $e_{\max}$ across platforms}
\label{tab:emax-platforms}
\vspace{-2mm}
\resizebox{0.48\textwidth}{!}{
\begin{tabular}{|c|c|c|c|c|}
\hline
\textbf{Platform} & \textbf{Precision} & \textbf{$e_{\max}$} & \textbf{$e_{\max}/u$} & \textbf{N-dependence} \\
\hline
CPU (Xeon) & FP64 & $6 \times 10^{-16}$ & $\sim 5$ & Constant \\
CPU (Xeon) & FP32 & $4 \times 10^{-7}$ & $\sim 6$ & Constant \\
\hline
GPU (H100) & FP64 & $1.0 \times 10^{-17}\sqrt{N} + 2.5 \times 10^{-16}$ & $3$--$6$ & $\propto \sqrt{N}$ \\
GPU (H100) & FP32 & $5.0 \times 10^{-9}\sqrt{N} + 1.2 \times 10^{-7}$ & $3$--$6$ & $\propto \sqrt{N}$ \\
GPU (H100) & BF16 & $8 \times 10^{-3}$ & $\sim 2$ & Constant \\
GPU (H100) & FP16 & $1 \times 10^{-3}$ & $\sim 2$ & Constant \\
\hline
NPU (910B) & BF16 & $8 \times 10^{-3}$ & $\sim 2$ & Constant \\
NPU (910B) & FP16 & $1 \times 10^{-3}$ & $\sim 2$ & Constant \\
NPU (910B) & FP32 & $2 \times 10^{-6} \sqrt{N/1024}$ & $\sim 34\sqrt{N/1024}$ & $\propto \sqrt{N}$ \\
\hline
\end{tabular}
}
\end{table}

Platform-specific observations explain these variations: on CPU, $e_{\max}$ is constant ($\sim 5$--$6u$) due to FMA instructions providing optimal rounding; on GPU for FP32/FP64, $e_{\max}$ grows with $\sqrt{N}$ (R$^2 > 0.88$) likely due to tensor core accumulation strategies; for GPU BF16/FP16, $e_{\max} \approx 2u$ is constant because GPU uses FP32 internal accumulation for low-precision operations; similarly, NPU BF16/FP16 shows constant $e_{\max} \approx 2u$ because internal accumulation uses FP32 with only output rounding; while NPU FP32 exhibits $e_{\max} \propto \sqrt{N}$ due to per-step rounding in accumulation.

\subsection{False Positive Rate}

V-ABFT achieves zero false positive rate across all tested distributions ($\mathcal{N}(10^{-6}, 1)$, $\mathcal{N}(1, 1)$, $\mathcal{U}(-1, 1)$, truncated $\mathcal{N}$) and all three precisions (BF16, FP16, FP32) over 100k trials per configuration. Both V-ABFT and A-ABFT maintain 0\% FPR when properly implemented, as the thresholds are designed to bound the maximum expected rounding error.

\subsection{Detection Rate}

\begin{table}[tb]
\centering
\caption{V-ABFT Detection Rate (\%) for BF16, Matrix Size (128, 1024, 256)}
\label{tab:detection-bf16}
\vspace{-2mm}
\resizebox{0.48\textwidth}{!}{
\begin{tabular}{|c|c|c|c|c|}
\hline
Bit & $\mathcal{N}(10^{-6}, 1)$ & $\mathcal{N}(1, 1)$ & $\mathcal{U}(-1, 1)$ & Truncated $\mathcal{N}$ \\
\hline
7 (exp LSB) & 0.01 & 0.00 & 19.66 & 10.90 \\
8 & 36.70 & 69.55 & 46.85 & 36.49 \\
9 & 73.48 & 100.00 & 75.03 & 99.38 \\
10 & 99.99 & -- & 99.86 & 99.96 \\
11 & 100.00 & 100.00 & 100.00 & 100.00 \\
12 & 100.00 & 100.00 & 100.00 & 100.00 \\
13 & 100.00 & 100.00 & 100.00 & 100.00 \\
14 & 100.00 & -- & 100.00 & 100.00 \\
\hline
\end{tabular}
}
\vspace{-2mm}
\end{table}

Table~\ref{tab:detection-bf16} shows V-ABFT achieves 100\% detection for bit positions 11--14 (the most impactful exponent bits) and $>$99\% for bit 10. Lower bits (7--9) have reduced detection rates as expected---these cause smaller magnitude changes that may fall within the rounding error margin.

\subsection{Scalability with Matrix Size}

\begin{table}[tb]
\centering
\caption{V-ABFT Detection Rate (\%) at Different Scales (BF16)}
\label{tab:scalability}
\vspace{-2mm}
\resizebox{0.48\textwidth}{!}{
\begin{tabular}{|c|c|c|c|c|}
\hline
\multirow{2}{*}{Bit} & \multicolumn{2}{c|}{(128, 4096, 256)} & \multicolumn{2}{c|}{(4096, 4096, 4096)} \\
\cline{2-5}
& $\mathcal{N}(10^{-6},1)$ & Trunc. $\mathcal{N}$ & $\mathcal{N}(10^{-6},1)$ & Trunc. $\mathcal{N}$ \\
\hline
9 & 39.86 & 97.46 & 0.00 & 67.54 \\
10 & 99.98 & 99.99 & 96.41 & 100.00 \\
11 & 100.00 & 100.00 & 100.00 & 100.00 \\
\hline
\end{tabular}
}
\end{table}

As matrix size increases, accumulated rounding errors grow, reducing detection sensitivity for lower bits. However, V-ABFT maintains 100\% detection for the most critical higher exponent bits even at $4096^3$ scale.

\subsection{Real Model Data}

We validated V-ABFT on actual model weights. For LLaMA-7B, we tested 111 weight matrices with 0\% FPR using V-ABFT thresholds. For GPT-2, we performed 5,379 GEMM verifications with 0\% FPR. For ViT-B/32, we conducted 50 epochs of fine-tuning with 5,937 sampled verifications, all achieving 0\% FPR.

\subsection{Performance Overhead}

When integrated into the FTAN-GEMM framework on Ascend 910B, V-ABFT threshold computation adds $<$2\% additional overhead, with total fault-tolerant GEMM overhead averaging 11.98\%. For comparison, DMR (Double Modular Redundancy) incurs $>$200\% overhead.

\section{Related Work}

\paragraph{ABFT origins.} Huang and Abraham~\cite{huang-abraham} introduced ABFT for matrix operations, enabling error detection through checksum encoding. Subsequent work extended this to parallel systems~\cite{abft-parallel} and fail-stop failures~\cite{abft-fail-stop}.

\paragraph{Threshold determination.} Early approaches used experimental calibration~\cite{exp-calibration} or conservative analytical bounds~\cite{higham-accuracy}. SEA-ABFT~\cite{sea-abft} simplified error analysis but remained loose. A-ABFT~\cite{a-abft} introduced probabilistic bounds but with restrictive assumptions.

\paragraph{ABFT for deep learning.} Recent work has adapted ABFT for GPU-accelerated training~\cite{ft-gemm-gpu}, Transformer attention~\cite{attnchecker}, and CNN convolutions~\cite{ft-cnn}. However, these focus on system integration rather than threshold methodology.

\paragraph{Mixed-precision challenges.} The shift to BF16/FP16 training~\cite{mixed-precision-training} has exposed limitations of FP64-era threshold methods, motivating our work.

\section{Conclusion}

We presented V-ABFT, a variance-based adaptive threshold algorithm for fault-tolerant matrix multiplication. By directly modeling verification differences through statistical variance estimation, V-ABFT achieves bounds within $7$--$20\times$ of actual rounding errors for FP32/FP64 and $48$--$158\times$ for BF16---a $6$--$48\times$ improvement over A-ABFT---while eliminating false positives across diverse distributions.

A key insight is that verification error depends on \emph{output precision}, not input precision. For fused-kernel ABFT that verifies before output quantization, low-precision GEMM can achieve FP32-level thresholds ($e_{\max} \approx 10^{-6}$), representing $\sim$1000$\times$ finer detection granularity compared to offline verification with low-precision output ($e_{\max} \approx 10^{-3}$). This finding enables detecting significantly smaller errors in mixed-precision training.

The algorithm's $O(n)$ complexity and platform-agnostic design make it practical for production deployment in mixed-precision deep learning systems.

Future work includes developing theoretical guarantees for the extrema-variance bound under specific distribution families, and integrating with emerging accelerator architectures.


\bibliographystyle{ACM-Reference-Format}

\end{document}